\documentclass{ieeetj}
\usepackage{comment}
\usepackage{cite}
\usepackage{amsmath,amssymb,amsthm,amsfonts}
\newtheorem{theorem}{Theorem}
\newtheorem{lemma}{Lemma}

\usepackage{algorithmic}
\usepackage{subcaption}
\usepackage{graphicx}
\usepackage{textcomp}
\usepackage[linesnumbered,ruled,vlined,commentsnumbered]{algorithm2e}
\usepackage{longtable}
\usepackage{lmodern}
\usepackage{float}
\usepackage{booktabs}
\usepackage{tabularx}
\newcolumntype{C}{>{\centering\arraybackslash}X}
\usepackage{hyperref}
\usepackage[normalem]{ulem}
\def\BibTeX{{\rm B\kern-.05em{\sc i\kern-.025em b}\kern-.08em
    T\kern-.1667em\lower.7ex\hbox{E}\kern-.125emX}}
\def\BibTeX{{\rm B\kern-.05em{\sc i\kern-.025em b}\kern-.08em
    T\kern-.1667em\lower.7ex\hbox{E}\kern-.125emX}}
\AtBeginDocument{\definecolor{tmlcncolor}{cmyk}{0.93,0.59,0.15,0.02}\definecolor{NavyBlue}{RGB}{0,86,125}}

\newcommand{\orcid}[1]{%
  \href{https://orcid.org/#1}{\includegraphics[height=5pt]{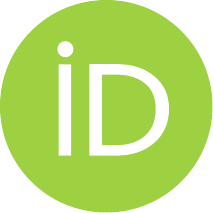}}%
}

\def\OJlogo{\vspace{-4pt}$<$Society logo(s) and publication title will appear here.$>$}
\def\seclogo{\vspace{10pt}$<$Society logo(s) and publication title will appear here.$>$}

\def\authorrefmark#1{\ensuremath{^{\textbf{#1}}}}

\begin{document}
\receiveddate{XX Month, XXXX}
\reviseddate{XX Month, XXXX}
\accepteddate{XX Month, XXXX}
\publisheddate{XX Month, XXXX}
\currentdate{XX Month, XXXX}
\doiinfo{XXXX.2022.1234567}

\markboth{}{Zantalis {et al.}}

\title{FedZMG: Efficient Client-Side Optimization in Federated Learning}

\author{Fotios Zantalis\authorrefmark{1,\orcid{0000-0001-7005-4264}}, Evangelos Zervas\authorrefmark{1,\orcid{0000-0002-7095-5123}},\\ and Grigorios Koulouras\authorrefmark{1,*,\orcid{0000-0002-1697-3670}}(Member, IEEE)}
\affil{TelSiP Research Laboratory, Department of Electrical and Electronic Engineering, School of Engineering, University of West Attica, Ancient Olive Grove Campus, 250 Thivon Str., GR-12241 Athens, Greece}
\corresp{Corresponding author: Grigorios Koulouras (e-mail: gregkoul@uniwa.gr).}
%\authornote{This paragraph of the first footnote will contain support information, including sponsor and financial support acknowledgment. For example, ``This work was supported in part by the U.S. Department of Commerce under Grant 123456.''}

\begin{abstract}
Federated Learning (FL) enables distributed model training on edge devices while preserving data privacy. However, clients tend to have non-Independent and Identically Distributed (non-IID) data, which often leads to client-drift, and therefore diminishing convergence speed and model performance. While adaptive optimizers have been proposed to mitigate these effects, they frequently introduce computational complexity or communication overhead unsuitable for resource-constrained IoT environments. This paper introduces Federated Zero Mean Gradients (FedZMG), a novel, parameter-free, client-side optimization algorithm designed to tackle client-drift by structurally regularizing the optimization space. Advancing the idea of Gradient Centralization, FedZMG projects local gradients onto a zero-mean hyperplane, effectively neutralizing the ``intensity'' or ``bias'' shifts inherent in heterogeneous data distributions without requiring additional communication or hyperparameter tuning. A theoretical analysis is provided, proving that FedZMG reduces the effective gradient variance and guarantees tighter convergence bounds compared to standard FedAvg. Extensive empirical evaluations on EMNIST, CIFAR100, and Shakespeare datasets demonstrate that FedZMG achieves better convergence speed and final validation accuracy compared to the baseline FedAvg and the adaptive optimizer FedAdam, particularly in highly non-IID settings.
\end{abstract}

\begin{IEEEkeywords}
Federated Learning, Distributed Learning, Zero Mean Gradients, client-side optimization, IoT, Cross-Device, Device Heterogeneity, client-drift
\end{IEEEkeywords}

%\IEEEspecialpapernotice{(Invited Paper)}

\maketitle

\section{INTRODUCTION}
\label{sec:introduction}
\PARstart{F}{ederated} Learning (FL) is a Distributed Machine Learning (DML) architecture focused on data privacy preservation \cite{mcmahan2017communication}. In FL, multiple distributed clients train a shared global model under the orchestration of a central server. These clients are often low-powered heterogeneous edge devices. On every training round, each participating client gets a version of the global model and trains it with their local data without sharing them with the rest of the clients or the central server. Instead, clients send the calculated weights back to the central server, where they are aggregated and used to update the global model. While FL offers significant benefits, it also faces a unique set of challenges, such as the high communication cost, and non-Independent and Identically Distributed (non-IID) data handling \cite{zhang2021survey, wen2023survey, li2020federated}. A crucial issue arising from the non-IID data is client-drift \cite{shi2022optimization}. When the data distribution varies significantly on each client, multiple local training steps cause models to drift towards local optima which could be very far away from a global optimum for the shared central model. Therefore, the central model may lose in performance and convergence speed. 
The use of adaptive optimizers has been proposed as a way to tackle client-drift and increase model performance and stability in FL \cite{reddi2020adaptive}. Adaptive optimizers use information from past gradients to adapt the Learning Rate (LR) and speed up or slow down updates for model weights. In FL there are usually two types of optimizers. An optimizer on the server side, where the model weights of individual clients are aggregated, and a client optimizer, used to update each client’ s local model weights. Adaptive optimizers like Adam, Adagrad, Yogi, and AdaDB have demonstrated promising results when selected as server-side optimizers \cite{reddi2020adaptive, zantalis2025data}. However, adaptive optimizers introduce additional hyperparameters that require careful tuning and need to be communicated during training. Therefore, an adaptive optimizer on the client side would significantly increase the communication cost in an FL system. There have also been suggestions for client-side adaptive optimization, where methods like FedCAda, or FedAdap implement adaptive techniques on the client side in order to improve stability, accelerate convergence and even decrease the needed communication rounds \cite{zhou2025fedcada, kundroo2023efficient}. 
While existing client-side adaptive optimization techniques offer significant potential to improve FL convergence and tackle client-drift, they usually introduce computational or communicational overhead that can be critical in real-life Internet of Things (IoT) applications with resource-constrained devices. Achieving convergence acceleration and stability while minimizing the data transmitted in every communication round can be a really challenging task. Therefore, there is an apparent need for communication-efficient, parameter-free client-side optimization techniques that can effectively mitigate heterogeneity without the complexity or overhead associated with stateful adaptive optimizers. In this paper, the algorithm Federated Zero Mean Gradients (FedZMG) is introduced. FedZMG is a novel client-side optimizer that requires minimum hyperparameter tuning and no additional communication overhead, presenting a significant advantage over classic adaptive optimizers like Adam. The core idea of FedZMG is to address non-IID data by constraining the optimization space directly on the client. In many heterogeneous settings, client-drift is observed to be increased by ``intensity'' or ``bias'' shifts in local gradients. Drawing inspiration from the Gradient Centralization (GC) technique originally proposed for centralized deep learning \cite{yong2020gradient}, FedZMG neutralizes these shifts by projecting local gradients onto a zero-mean hyperplane before every update. This mechanism acts as a powerful structural regularizer, forcing the optimizer to focus on the structural features of the data rather than client-specific biases. The entire mechanism takes place exclusively on the client devices with negligible computational cost, and no auxiliary variables are ever communicated to the aggregation server. The main goal of this research is to develop and validate a computationally light and communication-efficient client-side optimization algorithm that mitigates client-drift by design. The key contributions are the following: The proposal of the FedZMG algorithm, a parameter-free client-side optimizer that utilizes zero mean gradient projection to robustly handle non-IID data; the theoretical analysis providing the geometric properties of the ZMG operator and establishing the convergence rate of FedZMG in heterogeneous settings; and an extensive empirical evaluation of FedZMG demonstrating superior performance and stability compared to baseline algorithms such as FedAvg and FedAdam.
The rest of this paper is organized as follows. Section 2 presents related work in FL client-side optimization, adaptive optimizers, and client-drift mitigation. Section 3 provides a detailed description of the FedZMG algorithm. Section 4 presents the theoretical analysis. Section 5 demonstrates the experimental setup, the metrics used, and the results, and Section 5 discusses the empirical results. Finally, Section 6 wraps up the work in a conclusion, highlighting possible limitations, and outlining potential directions for future work. 

\section{BACKGROUND AND RELATED WORK}
\label{sec:background}

A significant amount of research has been conducted in recent years regarding FL optimization. Many researchers attempt to adapt methods and ideas from centralized training, in order to tackle FL unique challenges. The primary goal in FL is to mitigate the effects of data and resource heterogeneity like client-drift, while maintaining communication and computation efficiency. The applied optimization methods can be broadly categorized by whether they mainly modify the server-side, the client-side, or both sides of an FL system.
A straightforward approach is to apply adaptive optimizers on the server side. In such a setup, the clients use a standard Stochastic Gradient Descent (SGD) on their local data, while an adaptive optimizer like Adam, Yogi, or Adagrad is applied on the aggregated signals received from the clients on the server side \cite{reddi2020adaptive}. These methods have the benefit of not introducing any additional communication overhead, since the adaptive optimizer variables remain on the server and there is no need to exchange them with the clients.
FedCAda is an adaptive algorithm that uses an Adam-like optimizer on the client side to accelerate training and maintain stability \cite{zhou2025fedcada}. However, this method requires uploading and downloading Adam’s first and second moment estimates alongside the model weights, increasing therefore the per-round communication overhead.
A significant hurdle that many optimization techniques attempt to overcome is “client-drift”, the divergence of local optimum from the global optimum due to local data heterogeneity. Many client-side optimization methods have been developed to address this issue. SCAFFOLD uses local control variates, which are maintained by the clients and a global control variate, maintained by the server. During training, the local gradient is corrected by subtracting the local variate and adding the global variate \cite{karimireddy2020scaffold}. Other techniques, like Momentum Federated Learning (MFL), take advantage of the momentum functionality using an SGDm optimizer directly on the client-side to accelerate convergence \cite{mills2022client, liu2020accelerating}. However, both SCAFFOLD and MFL increase the communication overhead, as they require clients to send their optimizer states like momentum, or control variates, back to the server, in addition to the trained model weights. Federated Global Biased Optimizer (FedGBO) is using a momentum approach like MFL, but clients download a global momentum value, which remains fixed throughout the local training loop \cite{mills2022accelerating}. This method saves communication bandwidth as there are no additional values uploaded apart from the model weights. The AdaBest algorithm was proposed as a way to address similar trade-offs. AdaBest is an adaptive algorithm that subtracts client-drift estimates at the client side and adds an adaptive, corrective step on the server side after the aggregation. Additionally, the authors of AdaBest state that it reduces the communication overhead by about half compared to SCAFFOLD \cite{varno2022adabest}. FedLion is another adaptive algorithm designed specifically for communication efficiency. This algorithm adapts the Lion optimizer to FL. To minimize communication overhead, Lion transmits a compressed integer-valued vector, and one momentum vector, achieving faster convergence with only a small increase in communication volume compared to FedAvg \cite{tang2024fedlion}. FedAdap is an algorithm that performs adaptive hyper-parameter optimization on the client side during the local training loop \cite{kundroo2023efficient}. During each local epoch, the client’s training loss is monitored, and the client’s LR is dynamically reduced if the loss is not reduced below a specific threshold. Additionally, if the loss is persistently not improving for a specified number of epochs, the local training is interrupted earlier in order to save computational resources on the client. This method does not introduce any additional communication overhead, since only the model weights are sent back to the server.
Beyond algorithmic adaptations, fundamental research into gradient properties has also influenced distributed optimization. The researchers in \cite{yin2018gradient} introduced the concept of ``Gradient Diversity'' to quantify the dissimilarity between concurrent gradient updates, establishing a theoretical link between data heterogeneity and the scalability limits of distributed mini-batch SGD. In the realm of centralized deep learning, the idea of Gradient Centralization is proposed \cite{yong2020gradient}. A structural regularization technique that projects gradients onto a zero-mean hyperplane to improve training stability and generalization performance. These works provide the theoretical foundation for leveraging structural gradient properties to address optimization challenges in heterogeneous environments.

% ============================================================================
% FedZMG Updated Theory: Gradient Disagreement-Based Heterogeneity Detection
% ============================================================================

\section{PROPOSED ALGORITHM: FEDZMG}
\label{sec:proposedAlgorithm}

In this section, the basic FL problem is formulated, and the FedZMG algorithm is presented in detail, by providing a step by step pseudocode and an extended description. Additionally, the notation used throughout the paper is introduced.

\subsection{PROBLEM FORMULATION}
\label{subsec:problem}

FL is a method of solving an optimization problem of the form

\begin{equation}
\min_{\mathbf{w} \in \mathbb{R}^d} F(\mathbf{w}) = \sum_{k=1}^K p_k F_k(\mathbf{w})
\end{equation}

where $K$ is the total number of clients. $F_k(\mathbf{w})$ is the local objective function of the $k$-th client. The goal is to minimize a global objective function $F(\mathbf{w})$, which is a weighted average of the local objective functions. $p_k = \frac{n_k}{\sum_{j=1}^K n_j}$ is the relative weight of client $k$, where $n_j$ is the number of data samples held by client $j$. The local objective $F_k(\mathbf{w})$ is defined as the average error calculated from the client's local data distribution $\mathcal{D}_k$:

\begin{equation}
F_k(\mathbf{w}) = \frac{1}{n_k} \sum_{\xi \in \mathcal{D}_k} \ell(\mathbf{w}; \xi)
\end{equation}

where $\ell(\cdot)$ is the loss function for the data sample $\xi$. non-IID data is assumed, implying $\nabla F_i(\mathbf{w}) \neq \nabla F_j(\mathbf{w})$, which can lead to increased client-drift during local training. This heterogeneity is the primary challenge that FedZMG addresses through structural gradient regularization via zero mean gradients.

\subsection{ALGORITHM DESCRIPTION}
\label{subsec:algorithmdesc}

FedZMG uses a client-side optimizer which incorporates a Zero Mean Gradient (ZMG) operator directly into the local update rule. Unlike methods that require extra hyperparameters or server-side proxies, FedZMG addresses data heterogeneity by constraining the optimization space of the weight vectors.

The principal idea is that non-IID data is often expressed as ``intensity'' or ``bias'' shifts in the gradient vectors. The mean of the gradient vector carries this bias, while the variance carries the structural information \cite{yong2020gradient, ioffe2015batch, qiao2019micro, santurkar2018does}. By enforcing zero-mean gradients (centralization) before applying them, FedZMG forces the optimizer to focus on the structural features of the data, making the local updates more robust to client-specific shifts and reducing drift.

FedZMG algorithm can be described in three phases: I) Initialization, II) Local updates with Zero Mean Gradients, and III) Server Aggregation.

\subsubsection{PHASE I: INITIALIZATION}

At the beginning of every round $r$, the global model $\mathbf{w}^{(r)}$ is broadcast to a subset $K$ of clients $C$ named cohort size. The local parameters $\mathbf{w}_0^k= \mathbf{w}^{(r)}$ are initialized.

\subsubsection{PHASE II: LOCAL UPDATE WITH ZERO MEAN GRADIENTS}

For all local epochs $e \in \{1, \dots, E\}$, every client updates its model using a modified SGD rule that ensures gradients have sum values.

Let $\mathbf{w}_j^k$ denote the model parameters held by the $k$-th client at local step $j$ and
\begin{equation}
\mathbf{g}_j^k(\xi_j^k) = \nabla F_k(\mathbf{w}_j^k, \xi_j^k)
\end{equation}
the sample gradient of the local objective function $F_k(\cdot)$, where $\xi_j^k$ ia a data sample from the current batch. 

\noindent Zero Mean Operator: The ZMG operator is defined as 
\begin{equation}
\Phi_{ZMG}(\mathbf{g}) = \left(I-\frac{1}{d}\mathbf{1}\mathbf{1}^T\right) \mathbf{g} =
\mathbf{g}-\nu_{\mathbf{g}} \mathbf{1}
\end{equation}
where $\mathbf{1}$ is the $d$-dimensional all-ones vector and  
where $\nu_{\mathbf{g}} = \frac{1}{d} \mathbf{1}^T \mathbf{g}$.  In practice, for a weight matrix of a fully connected layer with shape $(C_{in}, C_{out})$, the zero-mean constraint is applied column-wise (over the input dimension $C_{in}$). For convolutional layers, it is applied over the spatial and channel dimensions.

The update rule at local step $j$ is defined as follows:

\begin{align}
\text{ZMG Projection:} \quad & \hat{\mathbf{g}}_j^k(\xi_j^k) = \Phi_{ZMG}(\mathbf{g}_j^k(\xi_j^k)) \\
\text{Weight Decay:} \quad & \hat{\mathbf{w}}_j^k = \mathbf{w}_j^k (1 - \lambda \eta) \\
\text{Update:} \quad & \mathbf{w}_{j+1}^k = \hat{\mathbf{w}}_j^k - \eta \cdot \hat{\mathbf{g}}_j^k(\xi_j^k)
\end{align}

where $\eta$ is the LR and $\lambda$ is the weight decay coefficient. Note that the weight decay is decoupled and applied directly to the weights, while the update is performed using the \textit{zero-mean} gradients.

Theoretical Justification: By removing the mean component of the gradient, FedZMG regularizes the weight space and reduces the Lipschitz constant of the loss function, leading to a smoother optimization landscape. In the context of FL, this makes local updates invariant to client-specific mean-shifts in the data distribution, thereby aligning the optimization trajectories of heterogeneous clients without reducing their LR.

\subsubsection{PHASE III: SERVER AGGREGATION}

The model parameters of the $k$-th client after updating them for $E$ epochs at round $r$ are denoted by $\mathbf{w}_k^{(r,E)}$. The clients transmit their updated models to the server  $ \mathbf{w}_k^{(r+1)}\gets\mathbf{w}_{k}^{(r,E)}$ and the server aggregates them using a standard weighted average similar to FedAvg:

\begin{equation}
\mathbf{w}^{(r+1)} = \sum_{k \in C} p_k \mathbf{w}_k^{(r+1)}
\end{equation}

FedZMG introduces no communication overhead compared to FedAvg, as the zero-mean operation is strictly local and no auxiliary variables are transmitted.

FedZMG steps are described in detail in Algorithm \ref{algo:FedZMG}.

\begin{algorithm}
\caption{FedZMG Pseudocode}
\label{algo:FedZMG}
\SetAlgoLined
\DontPrintSemicolon
\KwIn{$w^{(0)}$, $E$, $\eta$, $\lambda$}
\KwOut{Global model $w^{(T)}$}

\ForEach{round $r \in \{1, 2, \dots, T\}$}{
  $C \gets \text{SampleClients}(K)$\;
  \ForEach{client $k \in C$ \textbf{in parallel}}{
    \tcc{Initialize local model on client $k$}
    $w_0^k \gets w^{(t)}$\;
    
    \ForEach{epoch $e \in \{1, 2, \dots, E\}$}{
      \ForEach{batch $\mathcal{D}_{e,m}^k, m=1,\ldots |\mathcal{D}_e^k|$, and each   $\xi_j^k \in \mathcal{D}_{e,m}^k$}{
        $g_j^k \gets \nabla F_k(w_j^k; \xi_j^k)$\;
        
        \tcc{Apply Zero Mean Gradient Projection}
        \If{$g_j^k$ is weight matrix}{
            $\nu_{g} \gets \text{mean}(g_j^k, \text{axis}=\text{inputs})$\;
            $\hat{g}_j^k \gets g_j^k - \nu_{g}$\;
        }
        \Else{
            $\hat{g}_j^k \gets g_j^k$\;
        }
        
        \tcc{Apply decoupled weight decay and ZMG update}
        $w_{j+1}^k \gets w_j^k (1 - \lambda \eta) - \eta \cdot \hat{g}_j^k$\;
      }
    }
    $w_{k}^{(r+1)} \gets w_{k}^{(r,E)}$\;
  }
  $w^{(r+1)} \gets \sum_{k \in C} p_k w_{k}^{(r+1)}$\;
}
\Return $w^{(T)}$\;
\end{algorithm}

\section{THEORETICAL ANALYSIS}
\label{sec:theoryticalAnalis}

In this section, a theoretical analysis of FedZMG is presented. The basic geometric properties of the ZMG projection operator are presented and the convergence of FedZMG is shown under standard assumptions used in the literature. The results demonstrate that the algorithm mitigates the adverse effects of data heterogeneity by reducing the effective variance and gradient dissimilarity constants.

\subsection{GEOMETRIC PROPERTIES OF THE ZMG OPERATOR}
\label{subsec:geometricproperties}

Recall from the previous section, that if $\mathbf{g} \in \mathbb{R}^d$ is a gradient vector, the ZMG operator $\Phi_{ZMG}: \mathbb{R}^d \to \mathbb{R}^d$ projects it onto the hyperplane orthogonal to the all-ones vector $\mathbf{1}$:
\begin{eqnarray}
\Phi_{ZMG}(\mathbf{g}) &=& \mathbf{g} - \nu_{\mathbf{g}} \mathbf{1} = \mathbf{g} - \left( \frac{1}{d} \mathbf{g}^\top \mathbf{1} \right) \mathbf{1} \nonumber \\
&=& \left( I-\frac{1}{d}\mathbf{1}\mathbf{1}^T\right) \mathbf{g} =
\Phi \cdot \mathbf{g} 
\end{eqnarray}
where $\nu_{\mathbf{g}}$ is the scalar mean of the vector elements 
and $\Phi$ the corresponding projection matrix. 
It is clear that the matrix $\Phi$ is symmetric, $\Phi^T=\Phi$ and idempotent $\Phi^2 = \Phi$. Moreover, the the operator $\Phi_{ZMG}$ is non-expansive. That is, for any vector $\mathbf{x}$, $\|\Phi_{ZMG}(\mathbf{x})\| \le \|\mathbf{x}\|$. Furthermore, for any two heterogeneous gradients $\mathbf{g}_i$ and $\mathbf{g}_j$ with different mean intensities (biases), the distance between them is strictly reduced in the projected space:
\begin{equation}
\| \Phi_{ZMG}(\mathbf{g}_i) - \Phi_{ZMG}(\mathbf{g}_j) \|^2 = \| \mathbf{g}_i - \mathbf{g}_j \|^2 - d (\nu_i - \nu_j)^2
\end{equation}
where $\nu_i$ and $\nu_j$ are the scalar means of $\mathbf{g}_i$ and $\mathbf{g}_j$.

To see this, decompose the difference vector $\Delta = \mathbf{g}_i - \mathbf{g}_j$ into a mean component $\Delta_{\nu} = (\nu_i - \nu_j)\mathbf{1}$ and a zero-mean component $\Delta_{zmg} = \Phi_{ZMG}(\mathbf{g}_i) - \Phi_{ZMG}(\mathbf{g}_j)$.
By construction, $\Delta_{zmg} \perp \mathbf{1}$, and thus $\Delta_{zmg} \perp \Delta_{\nu}$.
By the Pythagorean theorem:
$$ \| \mathbf{g}_i - \mathbf{g}_j \|^2 = \| \Delta_{zmg} + \Delta_{\nu} \|^2 = \| \Delta_{zmg} \|^2 + \| \Delta_{\nu} \|^2 $$
Noting that $\| \Delta_{\nu} \|^2 = \| (\nu_i - \nu_j)\mathbf{1} \|^2 = d(\nu_i - \nu_j)^2$, rearranging terms yields the result.

\subsection{CONVERGENCE ANALYSIS OF FEDZMG}
\label{subsec:convergenceanalysis}
In this section the convergence analysis for FedZMG is presented under the Non-IID, setting using full device participation to reduce complexity.  The analysis
parallels to some extent the one provided by \cite{li2019convergence} for the FedAvg algorithm. To facilitate reading the same notation is used and the differences between the two algorithms are highlighted through similar lemmas. 

\subsubsection{NOTATION AND ALGORITHM}
Let $\mathbf{w}_t^k$ denote the model parameters held by the $k$-th client at time $t$ and  $\mathbf{g}_t^k(\xi_t^k)=\nabla F_k(\mathbf{w}_t^k,\xi_t^k)$ a sample gradient of the local objective function $F_k$. Moreover, $\mathcal{I}_E=\{nE | n=1,2,\ldots\}$ stands for the set of synchronization time instants.  Then, the model update equations of FedZMG can be described in compact form as 
\begin{eqnarray}
\mathbf{v}_{t+1}^k &=& \mathbf{w}_t^k-n_t \mathbf{q}_t^k(\xi_t^k)   \\
\mathbf{w}_{t+1}^k &=& \left\{ \begin{array}{ll}
\mathbf{v}_{t+1}^k & t+1 \notin \mathcal{I}_E \\
\sum_k p_k \mathbf{v}_{t+1}^k & t+1 \in \mathcal{I}_E \end{array} \right. 
\end{eqnarray}
where $\mathbf{q}_t^k(\xi_t^k)=\mathbf{g}_t^k(\xi_t^k)-\frac{1}{d}\langle \nabla F_k(\mathbf{w}_t^k, \xi_t^k),\mathbf{1}\rangle$
The temporary variable $\mathbf{v}_{t+1}^k$ is introduced to represent the direct SGD update from the current model parameters, whereas the form of $\mathbf{w}_{t+1}^k$ depends on whether $t+1 \in \mathcal{I}_E$ or not. Additionally to the 
\[
\mathbf{g}_t =  \sum_k p_k \mathbf{g}_t^k(\xi_t^k), \quad  \bar{\mathbf{g}}_t  =  \sum_k p_k \nabla F_k(\mathbf{w}_t^k)  = \mathbb{E}[\mathbf{g}_t]
\]
and the virtual sequences 
\[
\bar{\mathbf{w}}_t=\sum_k p_k \mathbf{w}_t^k, \quad 
\bar{\mathbf{v}}_t=\sum_k p_k \mathbf{v}_t^k
\]
the following are defined for convenience
\[
\mathbf{q}_t =  \sum_k p_k \mathbf{q}_t^k(\xi_t^k), \quad  \bar{\mathbf{q}}_t  =  \sum_k p_k \mathbf{q}_t^k = \sum_k p_k \mathbb{E} [\mathbf{q}_t^k(\xi_t^k)]
\]
From the above definitions it is clear that $\bar{\mathbf{w}}_{t+1}=\bar{\mathbf{v}}_{t+1}$ and 
$\bar{\mathbf{v}}_{t+1}=\bar{\mathbf{w}}_{t}-n_t \mathbf{q}_t$.
\subsubsection{ASSUMPTIONS}
The following assumptions are made throughout the analysis. 
\begin{itemize}
    \item {\bf Assumption 1.} All the objective functions $F_1, F_2, \ldots, F_K$ are L-smooth, that is 
    \[
    \| \nabla F_k(\mathbf{x}) -\nabla F_k(\mathbf{y})\|_{\ast} \le L \| \mathbf{x} -\mathbf{y}\|
    \]
    where $\| \cdot \|_{\ast}$ denotes the dual norm. Working with the self dual Euclidean norm it can be shown that an immediate consequence of L-smoothness is that 
    \[
    F_k(\mathbf{y}) \le F_k (\mathbf{x})+ \langle \nabla F_k(\mathbf{x}), \mathbf{y}-\mathbf{x}\rangle +\frac{L}{2} \| \mathbf{y}-\mathbf{x}\|^2
    \] 
    Setting $\mathbf{y}=\mathbf{x}-\frac{1}{2L}\nabla F_k(\mathbf{x}) $ in the previous inequality, the following is obtained
    \[ 
    F_k(\mathbf{y})< F_k(\mathbf{x})-\frac{1}{2L} \|\nabla F_k(\mathbf{x})\|^2 
    \] 
    that is, using the gradient method with step size $\frac{1}{2L}$ decreases the function value. 
    \item {\bf Assumption 2. } All the objective functions $F_1, F_2, \ldots, F_K$ are $\mu$-strongly convex, that is
    \[
    F_k(\mathbf{y}) \ge F_k (\mathbf{x})+ \langle \nabla F_k(\mathbf{x}), \mathbf{y}-\mathbf{x}\rangle +\frac{\mu}{2} \| \mathbf{y}-\mathbf{x}\|^2
    \]
    \item{\bf Assumption 3.} The variance of stochastic gradients in each client is bounded
    \[ 
    \mathbb{E}\| \nabla F_k(\mathbf{w}_t^k, \xi_t^k) -\nabla F_k(\mathbf{w}_t^k) \|^2  \le \sigma_k^2
    \]
    \item{\bf Assumption 4.} The expected squared norm of stochastic gradients is uniformly bounded, i.e., 
    \[ 
    \mathbb{E}\| \nabla F_k(\mathbf{w}_t^k, \xi_t^k) \|^2  \le G^2
    \]
    for all $k=1,\ldots, K$ and time instants $t$. 
\end{itemize}

\subsubsection{KEY LEMMAS AND PROOFS}
The following key lemmas  replace the corresponding ones in \cite{li2019convergence}. 
\begin{lemma}[One-Step Progress]
\label{lem:onestep}
If $n_t\le \frac{1}{4L}$ and $\mathbf{1}^T(\bar{\mathbf{w}}_0-\mathbf{w}^{\ast})=0$ then
\begin{eqnarray}
   \lefteqn{  \mathbb{E}[\|\bar{\mathbf{w}}_{t+1}-\mathbf{w}^{\ast}\|^2] \le
   (1-n_t \mu) \mathbb{E}[\|\bar{\mathbf{w}}_{t}-\mathbf{w}^{\ast}\|^2]} \nonumber \\  
    &&  
    +\mathbb{E}\left[2 \sum_k p_k \| \bar{\mathbf{w}}_{t}-\mathbf{w}_t^k\|^2 \right] +n_t^2\mathbb{E}[\|\mathbf{q}_{t}-\bar{\mathbf{q}}_t\|^2] \nonumber \\
    &&+ n_t^2 6L\Gamma   -
    n_t^2 \mathbb{E}\left[ \sum_k p_k \frac{1}{d} (\mathbf{1}^T \nabla F_k(\mathbf{w}_t^k))^2 \right]
\end{eqnarray}
where $\Gamma=F^{\ast}-\sum_k p_k F_k^{\ast}$
\end{lemma}
\begin{proof}
The proof follows the same steps as in \cite{li2019convergence} with sequences 
$\mathbf{g}_t$ and $\bar{\mathbf{g}}_t$ replaced by $\mathbf{q}_t$ and $\bar{\mathbf{q}}_t$ respectively. However, in the process of proof the term 
$n_t \sum_k p_k \langle \bar{\mathbf{w}}_t -\mathbf{w}^{\ast}, \frac{1}{d}\mathbf{1}\mathbf{1}^T \nabla F_k (\mathbf{w}_t^k) \rangle$ appears. 
This $\mathcal{O}(n_t)$ term will force convergence away from zero. To remedy this problem, an additional assumption is imposed on the algorithm initialization process so that the term is completely vanished. Notice that  the inner product of the error vector $\bar{\mathbf{w}}_t-\mathbf{w}^{\ast}$ and a multiple of the all ones vector $\mathbf{1}$ is eliminated if the error vector lies in the space orthogonal to $\mathbf{1}$, Therefore, the algorithm is initialized with a model parameter vector $\bar{\mathbf{w}}_0$ such that $\mathbf{1}^T\bar{\mathbf{w}}_0-\mathbf{1}^T\mathbf{w}^{\ast}$. Since $\mathbf{w}_t^k$ is updated  using zero-sum vectors this condition will hold for every $\bar{\mathbf{w}}_t$. Moreover, since structural information is not available for the gradients, the last non negative term is simply ignored hereafter in Lemma 1. 
\end{proof}
\begin{lemma}[Bounding variance]
\begin{equation}
\mathbb{E}[\| \mathbf{q}_t - \bar{\mathbf{q}}_t\|^2] \le \sum_k p_k^2 
\left(\sigma_k^2 - \frac{1}{d} \mathbf{1}^T \Sigma_g^k \mathbf{1}\right)
\end{equation}
\end{lemma}
\begin{proof}
 Clearly, 
\begin{eqnarray}
\mathbb{E}[\| \mathbf{q}_t - \bar{\mathbf{q}}_t\|^2] &\le &  \mathbb{E}\left[ \sum_k p_k^2 \| \mathbf{q}_t^k - \mathbf{q}_t^k(\xi_t^k) \|^2 \right] \nonumber \\
&=& 
\sum_k p_k^2 \mathbb{E}[\| \mathbf{q}_t^k - \mathbf{q}_t^k(\xi_t^k) \|^2] \nonumber 
\end{eqnarray}
However, $\mathbf{q}_t^k(\xi_t^k)=\Phi \mathbf{g}_t^k(\xi_t^k) $
and $\mathbf{q}_t^k=\Phi \mathbf{g}_t^k $
where $\Phi$ is the projecting matrix. Thus, 
\begin{eqnarray}
\lefteqn{
\mathbb{E}[\| \mathbf{q}_t^k - \mathbf{q}_t^k(\xi_t^k) \|^2]= 
\mathbb{E}[\| \Phi(\mathbf{g}_t^k(\xi_t^k)-\mathbf{g}_t^k)\|^2 ] } \nonumber \\
&=&tr(\Phi \mathbb{E}\left[(\mathbf{g}_t^k(\xi_t^k)-\mathbf{g}_t^k)(\mathbf{g}_t^k(\xi_t^k)-\mathbf{g}_t^k)^T \Phi^T \right] \nonumber \\
&=& tr(\Phi \Sigma_g^k \Phi^T)=tr(\Phi^2 \Sigma_g^k)=tr(\Phi\Sigma_g^k) \nonumber \\
&=& tr((I-\frac{1}{d} \mathbf{1}\mathbf{1}^T)\Sigma_g^k) =
 tr(\Sigma_g^k)-\frac{1}{d} tr(\mathbf{1}\mathbf{1}^T\Sigma_g^k) \nonumber \\
 &=& 
\sigma_k^2 -\frac{1}{d} \mathbf{1}^T\Sigma_g^k \mathbf{1}
\end{eqnarray}
where $\Sigma_g^k=Cov(\mathbf{g}_t^k(\xi_t^k))$. 
\end{proof}
%------------
\begin{lemma}[Bounding divergence]
Using Assumption 4. and suitably chosen non-increasing $n_t$ such that $n_t \le 2n_{t+E}$ 
\begin{equation}
\mathbb{E}\left[\| \sum_k p_k \| \mathbf{w}_t^k -\bar{\mathbf{w}}_t \|^2 \right]\le 
4n_t^2 (E-1)^2 G^2
\end{equation}
\end{lemma}
\begin{proof}
By bounding the divergence of $\mathbf{w}_t^k$ in a similar way to \cite{li2019convergence} the following is obtained
\begin{eqnarray}
\lefteqn{\mathbb{E}\left[ \sum_k p_k \| \mathbf{w}_t^k -\bar{\mathbf{w}}_t \|^2 \right]\le 4n_t^2 (E-1)^2 G^2} \nonumber \\
&&\hspace{-1cm}-\sum_k p_k (E-1) \sum_{j=t_0}^{t-1} n_j^2 \frac{1}{d} 
\mathbf{1}^T \mathbb{E}[\nabla F_k(\mathbf{w}_t^k, \xi_j^k) \nabla F_k(\mathbf{w}_t^k, \xi_j^k)^T] \mathbf{1} \nonumber
\end{eqnarray}
If the correlation matrix of the gradients  satisfies (uniformly)
\[
\mathbb{E}[\nabla F_k(\mathbf{w}_t^k,\xi_j^k) \nabla F_k(\mathbf{w}_t^k,\xi_j^k)^T] \succeq \lambda I
\]
then a sharper upper bound can be obtained for Lemma 3. However, this means that the gradients have at least $\lambda$ variance in every direction so that the randomness of 
$\nabla F_k(\mathbf{w}_j^k,\xi_j^k))$ excites all curvature directions. This condition is even stronger than the L-smoothness and $\mu$-strong convexity of the functions $F_k(\cdot)$. For the sake of a fair comparison  this condition will not be taken into further consideration.
\end{proof}

\subsubsection{THE MAIN THEOREM}

\begin{theorem}[Convergence of FedZMG]
Under Assumptions 1-4 for a non increasing $n_t$ of the form $n_t =\frac{\beta}{t+\gamma}$ with $\beta>0$ and $\gamma>0$,  FedZMG satisfies:
\begin{equation}
\lim_{T\rightarrow \infty}\mathbb{E}[F(\bar{\mathbf{w}}_T)] = F^{\ast}
\end{equation}
\end{theorem}

\begin{proof}
Let $\Delta_t = \mathbb{E}\|\bar{\mathbf{w}}_t - \mathbf{w}^{\ast}\|^2$. Substituting the results of  Lemma 2 and 3 in Lemma 1, the following recursion is obtained
\begin{equation}
\Delta_{t+1}= (1-n_t\mu) \Delta_{t} + n_t^2 C
\end{equation}
where
\[
C=6L \Gamma + 8(E-1)^2 G^2 +
\sum_k p_k^2 
\left(\sigma_k^2 - \frac{1}{d} \mathbf{1}^T \Sigma_g^k \mathbf{1}\right)
\]
Using induction $\Delta_t$ can be bound as 
\begin{equation}
\Delta_t \le \frac{\delta}{\gamma+t}
\end{equation}
provided that $n_0=\min\{\frac{1}{\mu}, \frac{1}{4L}\}$, $\gamma > E$,  $\beta > \frac{1}{\mu}$
and $\delta >\max\{ \frac{\beta^2 C}{\beta\mu-1}, \gamma \|\bar{\mathbf{w}}_0-\mathbf{w}^\ast\|^2 \}$.
Then by the smoothness of $F(\cdot)$ it is implied that 
\[
\mathbb{E}[F(\bar{\mathbf{w}}_t)]-F^{\ast} \le \frac{L}{2} \Delta_t 
\]
which tends to zero as $t$ increases.
\end{proof}

\subsubsection{DISCUSSION}
The convergence rate of FedZMG takes the same form $\mathcal{O}(1/T)$ as standard FedAvg. However, the constant $C$ is strictly smaller than the FedAvg constant $B$ derived in \cite{li2019convergence}. Therefore, it is expected that FedZMG will converge faster than FedAvg. The experiments described in the next section confirm this
claim.

% #############################
\section{EXPERIMENTAL SETUP}
\label{sec:experimentalsetup}

In this section FedZMG is evaluated alongside established FL algorithms across  different datasets and learning objectives. Specifically, performance is benchmarked against FedAvg, a standard baseline, and FedAdam, a widely used adaptive method. Furthermore, this section explains the experimental methodology, covering the hardware specifications, selected datasets, ML models, and the hyperparameter tuning process. 
All experiments were executed on an Nvidia DGX V100 Workstation, using a single Tesla V100 DGXS 32GB GPU (NVIDIA Corporation, Santa Clara, CA, USA). The algorithms and models were implemented using the TensorFlow Federated (TFF) framework \cite{tff}. To guarantee a fair assessment, all algorithms were tested on identical ML tasks, and statistical analysis was performed to verify the significance of the results.

\subsection{DATASETS AND MODELS}
\label{subsec:datasetsandmodels}

For the character recognition task, a federated version of the EMNIST dataset is used \cite{cohen2017emnist}, which includes 62 distinct classes spread across 3,383 clients. The data partitioning strategy splits examples rather than clients. As a result, all 3,383 clients appear in both the training set, containing 671,585 examples, and the testing set, which has 77,483 examples. A Convolutional Neural Network (CNN) is used for this dataset. The architecture includes two 5×5 convolutional layers: the first has 32 channels and the second has 64. Each layer is followed by a 2×2 max-pooling layer. After the convolutional stages, a fully connected layer with 512 units and ReLU activation is included, leading to a softmax output layer. The model has a total of 1,690,046 trainable parameters.

For the task of image classification the federated CIFAR100 dataset was selected \cite{krizhevsky2009learning}, containing 60,000 color images (32×32 pixels) across 100 classes. In the federated setup, the dataset splits into 500 training clients and 100 testing clients, with each client holding 100 images. The corresponding CNN architecture includes two main convolutional blocks and a dense classification head. The first block has two sequential 3×3 convolutional layers with 32 channels and ReLU, followed by a 2×2 max-pooling layer and a dropout layer with a rate of 0.25. The second block has a similar structure but adds depth, featuring two 3×3 convolutional layers with 64 channels and ReLU activation, followed by max-pooling and dropout layers. The feature maps are flattened and processed by a fully connected layer with 512 units and ReLU, followed by a dropout layer with a rate of 0.5. The final dense output layer has 100 units to match the CIFAR100 classes. This model has 2,214,532 parameters.

Next-Character-Prediction (NCP) is tested using the federated Shakespeare dataset \cite{caldas2018leaf}, which encompasses 715 clients. Data splitting occurs at the example level. Therefore, all 715 clients are included in both the training set, which contains 16,068 examples, and the testing set, which has 2,356 examples. This dataset supports the training of a Recurrent Neural Network (RNN). Variable-length word sequences are processed by an embedding layer that maps tokens to 256-dimensional vectors. These vectors are then input into a Gated Recurrent Unit (GRU) layer with 1,024 units, which outputs 1,024-dimensional vectors at each step. Finally, a dense layer with 66 nodes corresponds to the vocabulary size. This architecture requires 4,022,850 trainable parameters.

\subsubsection*{NON-IID DATA DISTRIBUTION ANALYSIS}
To quantify statistical heterogeneity, the local data of randomly sampled clients ($n=10$) is analyzed, from EMNIST, CIFAR100, and Shakespeare. Five metrics are used: Volume (total sample count), Label Diversity (unique classes/characters), Normalized Entropy (distribution diversity), Gini Coefficient (class imbalance), and KL Divergence (deviation from global distribution). Table \ref{tab:non_iid_summary} presents the global summary statistics. CIFAR-100 exhibits the most severe statistical heterogeneity (highest KL Divergence: $1.57$), while EMNIST displays the highest variance in sample volume ($\sigma=92.1$). Table \ref{tab:client_samples} details specific client profiles (top three most distinct), highlighting the range of local distributions, from highly balanced clients to those dominated by a single class.

\begin{table}[ht]
\centering
\caption{Summary of Non-IID Metrics (Mean $\pm$ Std Dev). CIFAR100 shows the highest deviation from the global distribution (KL Div), while Shakespeare has the largest vocabulary.}
\label{tab:non_iid_summary}
\resizebox{\columnwidth}{!}{
\begin{tabular}{lccccc}
\toprule
\textbf{Dataset} & \textbf{Volume} & \textbf{Label Diversity} & \textbf{Entropy} & \textbf{Gini} & \textbf{KL Div.} \\
\midrule
EMNIST     & $236 \pm 92$  & $55 \pm 4$   & $0.84 \pm 0.03$ & $0.58 \pm 0.03$ & $0.24 \pm 0.09$ \\
CIFAR-100  & $100 \pm 0$   & $21 \pm 5$   & $0.56 \pm 0.07$ & $0.89 \pm 0.04$ & $1.57 \pm 0.28$ \\
Shakespeare& $3267 \pm 4030$ & $49 \pm 12$ & $0.80 \pm 0.04$ & --- & --- \\
\bottomrule
\end{tabular}
}
\end{table}

\begin{table}[ht]
\centering
\caption{Detailed metrics for representative clients. ``Dom. Class \%'' indicates the percentage of local data belonging to the most frequent class/character.}
\label{tab:client_samples}
\resizebox{\columnwidth}{!}{
\begin{tabular}{lrrrrr}
\toprule
\textbf{Client ID} & \textbf{Volume} & \textbf{Label Diversity} & \textbf{Entropy} & \textbf{KL Div.} & \textbf{Dom. Class. \%} \\
\midrule
\multicolumn{6}{l}{\textbf{\textit{EMNIST (High Volume Variance)}}} \\
f0698\_45 & 374 & 60 & 0.87 & 0.13 & 10.2\% \\
f0952\_49 & 328 & 58 & 0.88 & 0.31 & 6.7\% \\
f2295\_67 & 131 & 49 & 0.82 & 0.30 & 9.9\% \\
\midrule
\multicolumn{6}{l}{\textbf{\textit{CIFAR100 (High Class Imbalance)}}} \\
Client 54 & 100 & 29 & 0.68 & 0.98 & 10.0\% \\
Client 423& 100 & 18 & 0.56 & 1.28 & 15.0\% \\
Client 435& 100 & 18 & 0.43 & 1.73 & 48.0\% \\
\midrule
\multicolumn{6}{l}{\textbf{\textit{Shakespeare (High Sample Variance)}}} \\
K\_LEAR... & 11993 & 63 & 0.76 & --- & 17.8\% \\
K\_HENRY... & 5965 & 63 & 0.76 & --- & 20.0\% \\
ALL\_S...   & 245 & 38 & 0.78 & --- & 32.2\% \\
\bottomrule
\end{tabular}
}
\end{table}

\subsection{TRAINING SETUP AND HYPERPARAMETER TUNING}
\label{subsec:trainingsetup}

\subsubsection{LEARNING RATES}
All three federated algorithms under test have a set of two LRs that need to be tuned before the actual experiment execution. The client optimizer’s LR ($\eta_c$) and the server optimizer’s LR ($\eta_s$). A grid search on a 9 $\times$ 9 grid was performed to find the best client--server LR  pair for every dataset \cite{wang2021field}. LR values are logarithmically spaced between $10^{-3}$ and $10^1$. To maximize performance across diverse scenarios, each algorithm went through independent fine-tuning for every learning task. A grid search was conducted by training for 50 rounds using distinct LR pairs, with the optimal configuration selected based on the average validation accuracy of the final 10 rounds. During tuning, the cohort size was fixed at 10 and the epoch size at 4. Visualizations of the grid search results are provided as heat-maps in Figure \ref{fig:all_gridsearches}.
\begin{figure*}[!t]
    \centering
    \includegraphics[width=1.0\textwidth]{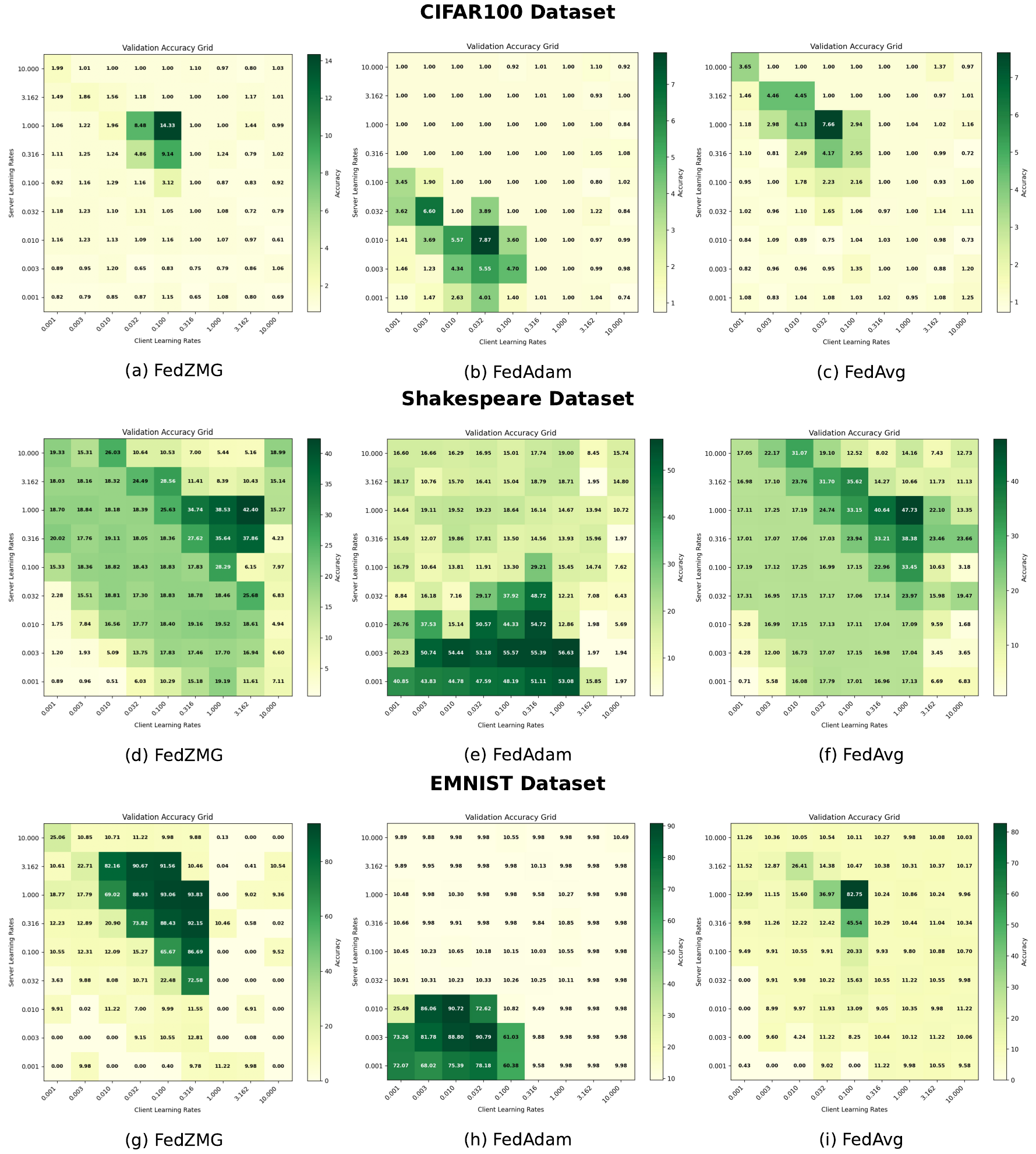}
    
    \caption{Client--server ($\eta_c-\eta_s$) learning rate tuning via grid search across three datasets: CIFAR100 (top), Shakespeare (middle), and EMNIST (bottom).}
    \label{fig:all_gridsearches}
\end{figure*}

\subsubsection{OTHER HYPERPARAMETERS}

FedZMG uses momentum and decoupled weight decay. The momentum was set to $m=0.9$ and weight decay $\lambda=0.0005$ based on empirical results. For FedAdam $\beta_1,\beta_2$, and $\epsilon$ were fixed to $\beta_1=0.9, \beta_2=0.99, \epsilon=0.001$ throughout the experiments \cite{reddi2020adaptive}. The number of local training rounds for every client was set to $epochs\ E=4$. Additionally, for the algorithm comparison experiments, the number of clients sampled on every round was set to $cohort\ size\ C=10$.

For the server aggregation method, a weighted average was used, weighting by the number of examples, for all algorithms. Finally, the training process was iterated over 1000 rounds for all the executed tests on all datasets.

\subsection{EVALUATION METRICS}
\label{subsec:evaluationmetrics}
Three distinct metrics based on validation accuracy are employed to ensure a comprehensive evaluation. First, ``Final Accuracy'', the average accuracy over the final 100 rounds, is used to quantify end-state performance. Second, convergence velocity is measured by the ``Round to Threshold'', defined as the number of rounds required for a 4-round moving average to consistently surpass specific targets (EMNIST: 75\%/85\%; CIFAR100: 20\%/33\%; Shakespeare: 35\%/45\%). Third, ``Late-stage Sustained Performance'' is computed by averaging accuracy after the slowest algorithm reaches the threshold. This approach isolates post-convergence stability while mitigating interval-length bias. To assess statistical significance, paired $t$-tests are conducted on results from five independent runs using identical data splits. Assuming a null hypothesis $H_0: \mu_D = 0$, the test statistic is calculated as:

\begin{equation}
t = \frac{\overline{D}}{s_D / \sqrt{n}}
\end{equation}

where $\overline{D}$ denotes the mean of the paired differences, $s_D$ the standard deviation, and $n$ the number of pairs. Significance is determined at the $\alpha = 0.05$ level.

\section{RESULTS}
\label{sec:results}

In this section, the results of the executed tests are presented and described in detail.

\subsection{ALGORITHMS COMPARISON}

Figure \ref{fig:algorithms_accuracy_comparison} demonstrates three plots with the validation accuracy of the three algorithms in 1000 communication rounds for every dataset. Five runs were performed per algorithm. The main curve represents the mean validation accuracy and the shaded areas show the standard deviation.

\begin{figure}[!t]
\centering

\subfloat[CIFAR100]{
\includegraphics[width=\linewidth]{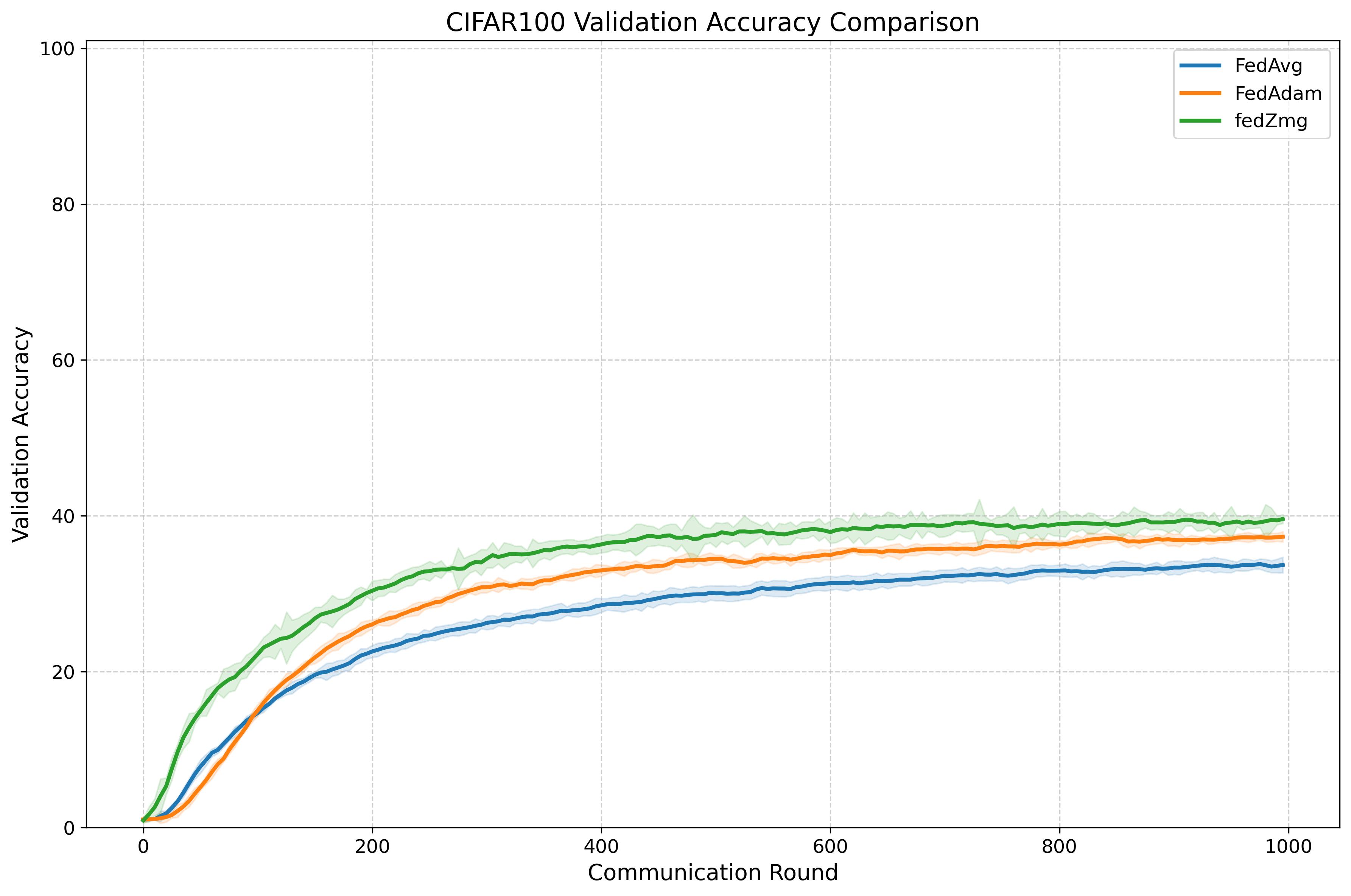}
}

\subfloat[Shakespeare]{
 \includegraphics[width=\linewidth]{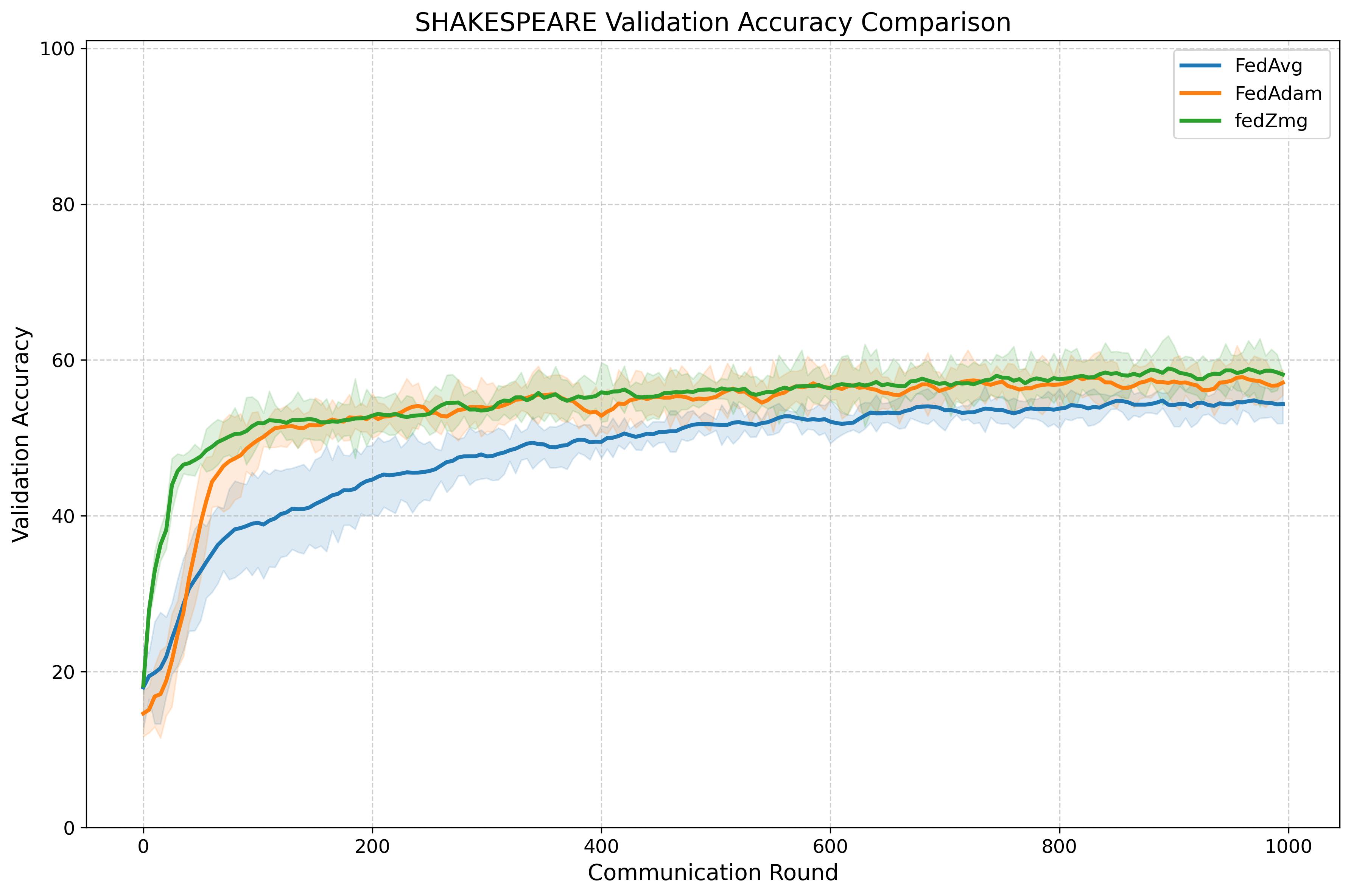}
}

\subfloat[EMNIST]{
\includegraphics[width=\linewidth]{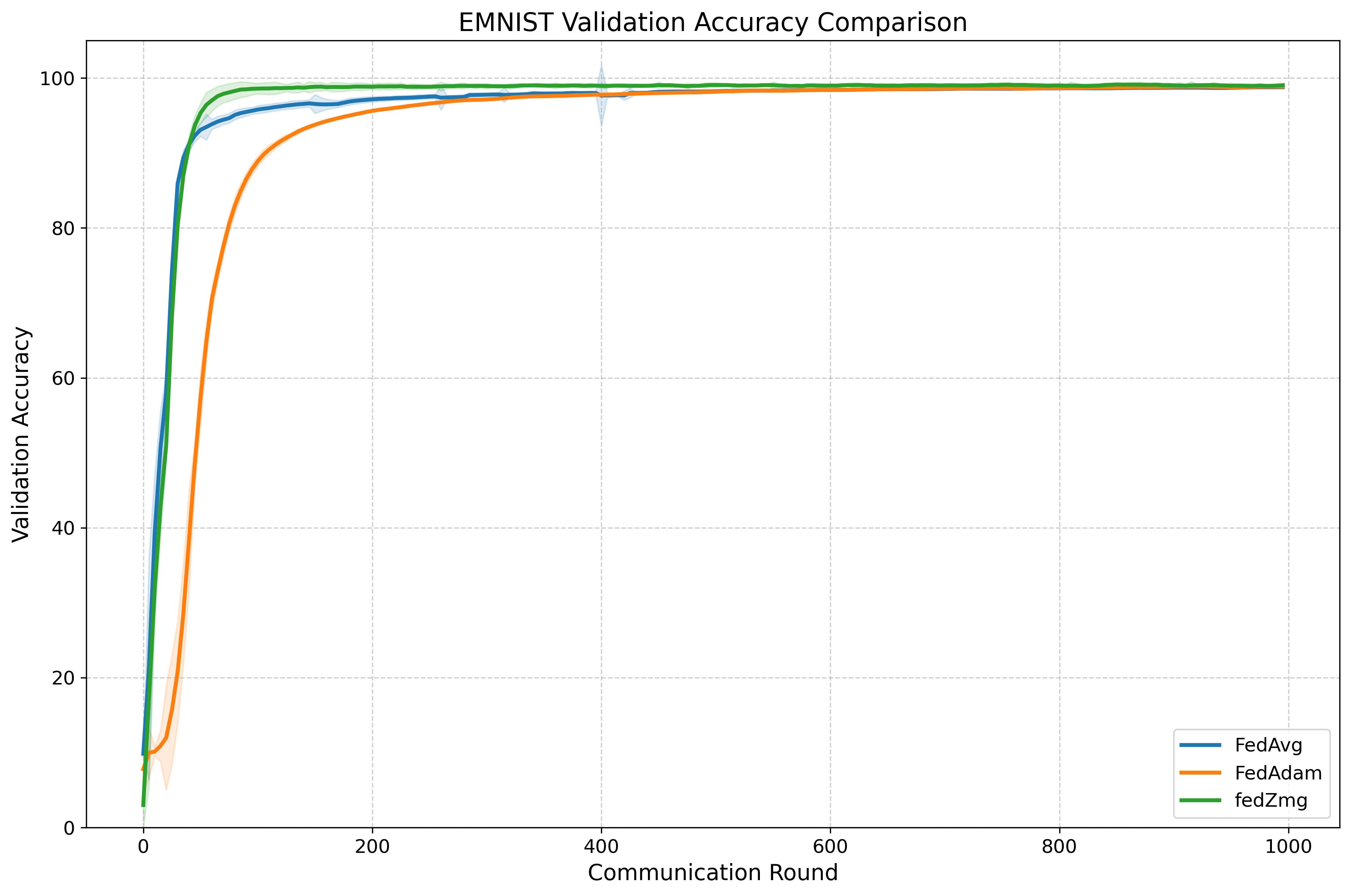}
 }

 \caption{Validation accuracy comparison. Top: CIFAR100; Middle: Shakespeare; Bottom: EMNIST.}
 \label{fig:algorithms_accuracy_comparison}
\end{figure}

Tables \ref{tab:cifar100}--\ref{tab:emnist} are summary tables per dataset, presenting the results of the selected evaluation metrics. The ``Rounds to Threshold'' columns show how many rounds it takes for each algorithm to reach the specified threshold. This offers valuable input about early convergence at two different stages. The column, ``Avg Acc post-Threshold'' presents the average validation accuracy of each algorithm after every algorithm reached the specified threshold. This helps with the interpretation of the algorithm's stability after reaching a satisfactory accuracy percentage.

\begin{table}[h]
\caption{Metrics comparison on CIFAR100 dataset.\label{tab:cifar100}}
\begin{tabularx}{\linewidth}{lCCC}
\toprule
\textbf{Algorithm} & \textbf{Rounds to} \newline \textbf{20.0\%} & \textbf{Rounds to} \newline \textbf{33.0\%} & \textbf{Avg Acc} \newline \textbf{Post-33.0\%} \\
\midrule
FedZMG & 95 & 250 & 39.50\% \\
FedAvg & 155 & 805 & 34.69\% \\
FedAdam & 135 & 420 & 37.68\% \\
\bottomrule
\end{tabularx}
\end{table}
\vspace{-1em}

\begin{table}[h]
\caption{Metrics comparison on Shakespeare dataset.\label{tab:shakespeare}}
\begin{tabularx}{\linewidth}{lCCC}
\toprule
\textbf{Algorithm} & \textbf{Rounds to} \newline \textbf{20.0\%} & \textbf{Rounds to} \newline \textbf{33.0\%} & \textbf{Avg Acc} \newline \textbf{Post-33.0\%} \\
\midrule
FedZMG & 15 & 15 & 55.79\% \\
FedAvg & 15 & 55 & 51.92\% \\
FedAdam & 30 & 45 & 56.62\% \\
\bottomrule
\end{tabularx}
\end{table}

\begin{table}[h]
\caption{Metrics comparison on EMNIST dataset.\label{tab:emnist}}
\begin{tabularx}{\linewidth}{lCCC}
\toprule
\textbf{Algorithm} & \textbf{Rounds to} \newline \textbf{20.0\%} & \textbf{Rounds to} \newline \textbf{33.0\%} & \textbf{Avg Acc} \newline \textbf{Post-33.0\%} \\
\midrule
FedZMG & 15 & 20 & 98.91\% \\
FedAvg & 15 & 20 & 98.43\% \\
FedAdam & 35 & 40 & 97.96\% \\
\bottomrule
\end{tabularx}
\end{table}

\subsection{STATISTICAL SIGNIFICANCE TESTS}

In this subsection, the results of the paired $t$-test are presented in order to evaluate the importance of the aforementioned results. For every dataset, the two pairs under test are FedZMG vs. FedAdam and FedZMG vs. FedAvg. The results are concentrated on Tables \ref{tab:emnist_ttest}--\ref{tab:cifar100_ttest}.

\begin{table}[h]
\caption{Paired $t$-test results on EMNIST dataset.\label{tab:emnist_ttest}}
\vspace{0.5em}
\begin{tabularx}{\linewidth}{l>{\centering\arraybackslash}X>{\centering\arraybackslash}X}
\toprule
\textbf{Comparison} & \textbf{\emph{t}-Statistic} & \textbf{\emph{p}-Value} \\
\midrule
FedZMG vs. FedAdam & 2.043 & $1.105 \times 10^{-1}$ \\
FedZMG vs. FedAvg & 4.025 & $1.579 \times 10^{-2}$ \\
\bottomrule
\end{tabularx}
\end{table}

\begin{table}[h]
\caption{Paired $t$-test results on Shakespeare dataset.\label{tab:shakespeare_ttest}}
\vspace{0.5em}
\begin{tabularx}{\linewidth}{l>{\centering\arraybackslash}X>{\centering\arraybackslash}X}
\toprule
\textbf{Comparison} & \textbf{\emph{t}-Statistic} & \textbf{\emph{p}-Value} \\
\midrule
FedZMG vs. FedAdam & 1.452 & $2.201 \times 10^{-1}$ \\
FedZMG vs. FedAvg & 3.854 & $1.824 \times 10^{-2}$ \\
\bottomrule
\end{tabularx}
\end{table}

\begin{table}[h]
\caption{Paired $t$-test results on CIFAR100 dataset.\label{tab:cifar100_ttest}}
\vspace{0.5em}
\begin{tabularx}{\linewidth}{l>{\centering\arraybackslash}X>{\centering\arraybackslash}X}
\toprule
\textbf{Comparison} & \textbf{\emph{t}-Statistic} & \textbf{\emph{p}-Value} \\
\midrule
FedZMG vs. FedAdam & 7.804 & $1.455 \times 10^{-3}$ \\
FedZMG vs. FedAvg & 20.716 & $3.208 \times 10^{-5}$ \\
\bottomrule
\end{tabularx}
\end{table}

%%%%%%%%%%%%%%%%%%%%%%%%%%%%%%%%%%%%%%%%%%
\section{DISCUSSION}
\label{sec:discussion}

\subsection{ANALYZING THE RESULTS}

Analyzing the empirical results of FedZMG across different learning tasks, it is apparent that FedZMG demonstrates a consistent improvement in both convergence speed and final model validation accuracy, compared to the baseline FedAvg. This advantage is more evident in scenarios with high statistical heterogeneity, such as the CIFRAR100 dataset. As presented in Table \ref{tab:non_iid_summary}, CIFAR100 demonstrates the highest KL Divergence ($1.57 \pm 0.28$), indicating high non-IID data. In this scenario, FedZMG reached the 33\% accuracy threshold in 250 rounds, more than three times faster than FedAvg with 805 rounds and significantly faster than FedAdam with 420 rounds.

This improvement can be explained by the theoretical properties of FedZMG derived the theoretical analysis section. Many FL algorithms are sensitive to ``intensity'' or ``bias'' shifts in local data distributions, which can be translated to large mean components in the local gradient vectors. By applying the ZMG operator, FedZMG filters out these client-specific biases. FedZMG operates in a constrained optimization space with a strictly smaller variance constant. As a result, the local updates are closer to the global optimum, reducing client-drift and allowing a more aggressive LR without destabilizing the global model.

It is worth mentioning that FedZMG outperformed FedAdam in the complex CIFAR100 task, although FedAdam is an adaptive optimizer correcting noisy gradients on the server side using historical moment estimations. FedZMG on the other hand, address the source of the noise directly on the client side. By structurally regularizing the gradients via centering before they are communicated to the aggregation server. This suggests that structural regularization of the gradient space and historical adaptation of the LR handle heterogeneity differently. However, FedZMG achieves this robustness without the communication overhead of transmitting auxiliary momentum vectors or the need to fine-tune additional hyperparameters like $\beta_1$ and $\beta_2$, which are mandatory for FedAdam.

FedZMG shows a robust behavior on both computer vision tasks and Natural Language Processing (NLP). In the Shakespeare dataset, FedZMG achieved a competitive post-threshold accuracy of 55.79\%. In the context of embedding layers and GRUs, the ``mean shift'' neutralized by ZMG can be interpreted as a form of instance-specific normalization. Similarly to how Batch Normalization (BN) reduces internal covariate shift in centralized training, the ZMG operator seemingly prevents local embedding updates from being dominated by high-frequency tokens or varying sentence lengths specific to a individual client. This allows the optimization process to focus on the structural features of the text, leading to better generalization across the heterogeneous corpus of distinct speaking roles in the Shakespeare dataset.

\subsection{STATISTICAL SIGNIFICANCE}

The statistical significance tests further support the experimental results of the previous sections, specifically regarding the improvement over the baseline FedAvg. The paired t-tests presented in Tables \ref{tab:emnist_ttest}-\ref{tab:cifar100_ttest} reject the null hypothesis for the FedZMG vs. FedAvg comparison across all datasets ($p < 0.05$). Specifically, for the CIFAR100 dataset, the p-value for the comparison between FedZMG and FedAvg is $3.208 \times 10^{-5}$, indicating a high level of statistical significance.

This statistical validation reinforces the theoretical convergence bounds established in Section \ref{subsec:convergenceanalysis}. The convergence rate of FL is heavily determined by the variance terms in the error bound. The current theoretical analysis showed that the constant factor $B_{ZMG}$ controlling the error bound is proportional to the reduced heterogeneity $G_{ZMG}^2$. The statistically significant margins in validation accuracy, especially in the early and middle stages of training (as seen in the ``Rounds to Threshold'' metrics), provide empirical validation that $B_{ZMG}$ is indeed significantly smaller than $B_{FedAvg}$ in practice. In the EMNIST and Shakespeare datasets, while the distinction from the adaptive FedAdam was less statistically pronounced ($p > 0.05$), FedZMG still maintained a statistically significant advantage over the standard FedAvg ($p < 0.05$), confirming that the reduction in gradient variance provides a consistent benefit over standard aggregation methods even in varied learning tasks.

\subsection{LIMITATIONS AND FUTURE WORK}

FedZMG shows a robust performance, supporting its theoretical guarantees. However, there are still some design limitations that would benefit from an extended investigation. The main idea behind the algorithm is that the mean gradient can be considered a bias or noise and not necessarily a learning signal. Although the claim that global features such as variance have high priority for many classification and recognition tasks, there can be other regression problems or specific model architectures where the global intensity or the mean gradient value hold valuable information as well. In these corner cases, projecting the gradient onto a zero-mean hyperplane could potentially slow down learning process. Additionally, the interaction between the ZMG operator and other normalization layers, such as BN presents a complex dynamic. As activations are also centered by BN, application of ZMG translates to essentially repeating a similar operation to the gradients. Therefore, extremely deep networks could see diminishing returns because the structural constraints can become too redundant. Furthermore, the experiments conducted show the robust advantages of FedZMG over baseline algorithms such as FedAvg and FedAdam, however it would be beneficial to investigate FedZMG's performance against other client-side optimization algorithms.
To address these limitations, future work could focus on extended experiments, comparing FedZMG with other client-side algorithms like SCAFFOLD or FedCAda over more learning tasks while using deeper and more complex model architectures. Additionally, it would be interesting to investigate the performance of hybrid optimization algorithms combining FedZMG's client-side benefits with an adaptive optimization technique on the server-side. A ``FedZMG-Adam'' algorithm might be able benefit from the variance reduction in ZMG to make Adam more stable.

\section{CONCLUSION}
\label{sec:conclusion}

In this paper, FedZMG was introduced. A parameter-free client-side FL algorithm designed to mitigate client-drift without introducing additional communication overhead. By projecting local updates onto a zero-mean hyperplane, FedZMG structurally reduces gradient heterogeneity, offering a theoretical guarantee for improved convergence in non-IID scenarios. Extensive experiments on EMNIST, CIFAR100, and Shakespeare datasets demonstrate that FedZMG consistently outperforms FedAvg and FedAdam in convergence speed and validation accuracy, particularly under conditions of high statistical heterogeneity. Future research will investigate the algorithm's interaction with deep network normalization layers and explore hybrid architectures that combine FedZMG with server-side adaptive optimization.

\section*{APPENDIX}
Appendixes, if needed, appear before the acknowledgment.

\section*{ACKNOWLEDGMENT}
The preferred spelling of the word ``acknowledgment'' in
American English is without an ``e'' after the ``g.'' Use the
singular heading even if you have many acknowledgments.
Avoid expressions such as ``One of us (S.B.A.) would like
to thank . . . .'' Instead, write ``F. A. Author thanks . . . .'' In
most cases, sponsor and financial support acknowledgments
are placed in the unnumbered footnote on the first page, not
here.

\bibliographystyle{IEEEtran}
\bibliography{bibliography}

\begin{IEEEbiography}[{\includegraphics[width=1in,height=1.25in,clip,keepaspectratio]{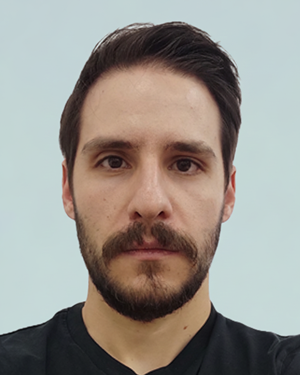}}]{Fotios Zantalis}{\,} was born in Athens, Greece in 1991. He received the BSc in Electronics Engineering and an MSc in Advanced Network Technologies and Computing from the University of West Attica, in 2014 and 2018, respectively. He is currently studying for the Ph.D. degree at the Department of Electrical and Electronic Engineering, University of West Attica. His current research interests are Federated Learning, Machine Learning, Internet of Things, and the Blockchain Technology. Contact him at fzantalis@uniwa.gr
\end{IEEEbiography}
%\vspace{-1cm}

\begin{IEEEbiography}[{\includegraphics[width=1in,height=1.25in,clip,keepaspectratio]{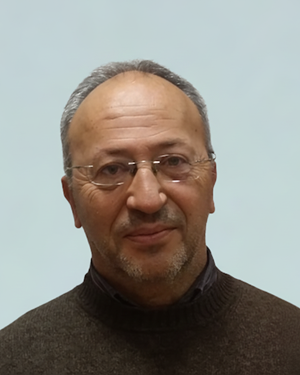}}]{Evangelos Zervas}{\,} received the Diploma in Electrical Engineering from the National Technical University of Athens, Greece, in 1987. He obtained the M.Sc. and Ph.D. degrees in 1989 and 1993, respectively, both in Electrical and Computer Engineering, from the Northeastern University, Boston, MA, USA. During the period January 1994 – July 1996, he served as a teaching assistant at the University of Athens, Dept. of Informatics and Telecommunications. In 1995, he joined the Faculty of the Technological Education Institute of Athens, Greece, and since 2018 he is a Professor at the Department of Electrical and Electronics Engineering, University of West Attica, Greece. Contact him at ezervas@uniwa.gr
\end{IEEEbiography}
%\vspace{-1cm}

\begin{IEEEbiography}[{\includegraphics[width=1in,height=1.25in,clip,keepaspectratio]{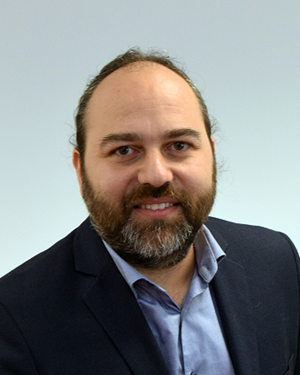}}]{Grigorios E. Koulouras}{\,} was born in Athens, Greece, in 1978. He received the B.Sc. degree in Electrical and Electronics Engineering from the Technological Educational Institute of Athens, Greece, in 1999, and the M.Sc. degree in Data Communication Systems from Brunel University, U.K., in 2001. He received the Ph.D. degree from the School of Engineering and Design, Brunel University, in 2008. He is currently an Associate Professor of Cloud Computing and Internet of Things with the Department of Electrical and Electronics Engineering, University of West Attica, Athens, Greece. From January 2018 to January 2022, he served as a Member of the Plenary of the Hellenic Telecommunications and Post Commission (EETT). He is also the Director and a co-founder of the “Telecommunications, Signal Processing and Intelligent Systems Research Laboratory (TelSiP)”. He has authored or co-authored more than 70 publications in peer-reviewed journals and conferences. His research interests include embedded systems, IoT architectures, and Federated Learning. Contact him at gregkoul@uniwa.gr
\end{IEEEbiography}

\vfill\pagebreak

\end{document}